\documentclass[letterpaper, 10 pt, journal, twoside]{ieeetran} 

\usepackage{times}
\IEEEoverridecommandlockouts                            
\usepackage{booktabs}
\usepackage{xcolor,colortbl}
\definecolor{lavender}{rgb}{0.9, 0.9, 0.98}
\usepackage{bbding}
\usepackage{pgf}
\usepackage{cite}
\usepackage{verbatim}
\usepackage{amsmath}
\usepackage{algorithm}
\usepackage[noend]{algpseudocode}
\usepackage{graphicx}
\usepackage{amssymb}
\usepackage{amsmath}
\usepackage{amsthm}
\usepackage{array}
\usepackage[draft]{fixme}
\usepackage[english]{babel}
\usepackage{verbatim}
\usepackage{tikz}
\usepackage{multirow}
\usepackage{xspace}
\usepackage{color}
\usepackage{tikz}
\usepackage{svg}
\usepackage{caption}
\usepackage{subcaption}
\usepackage{svg}
\usepackage[colorinlistoftodos]{todonotes}
\usepackage[font=small,labelfont=bf,tableposition=top]{caption}
\usepackage[normalem]{ulem}
\usepackage{pifont}
\newcommand\LyonsStrikeout{\bgroup\markoverwith
{\textcolor{cyan}{\rule[0.5ex]{2pt}{0.8pt}}}\ULon}
\usepackage{bm}
\usepackage{url}
\usepackage{bbm}
\usepackage[nolist]{acronym}

\makeatletter
\def\BState{\State\hskip-\ALG@thistlm}
\makeatother                           

\DeclareCaptionLabelFormat{andtable}{#1~#2  \&  \tablename~\thetable}

\usepackage{booktabs, tabularx}
\usepackage{multirow}

\usepackage{xspace}

\newtheorem{theorem}{Theorem}

\newtheorem{problem}{Problem}

\newtheorem{remark}{Remark}

\DeclareMathOperator*{\minimize}{\text{minimize}}

\DeclareMathOperator*{\argmin}{arg\,min}

\title{Distributed Lloyd-based algorithm for uncertainty-aware multi-robot
under-canopy flocking}
\author{{Manuel Boldrer, V\'it  Kr\'atk\'y, Viktor Walter, Martin Saska} 
\thanks{Manuscript received: September, 3, 2025; Revised December, 4, 2025;
Accepted February, 27, 2026.}
   \thanks{Manuel Boldrer,  V\'it  Kr\'atk\'y, Viktor Walter, and Martin Saska are with the Department of Cybernetics, Czech  Technical University in Prague, Karlovo namesti 13, 121 35 Prague 2, Czechia, {\tt
       \{name.surname\}@fel.cvut.cz}.}
       \thanks{This paper was recommended for publication by Editor Giuseppe
       Loianno upon evaluation of the Associate Editor and Reviewers’ comments.
       This work was funded by the Czech Science Foundation (GAČR) under
       research project no. 23-07517S, by the European Union under the project
       Robotics and advanced industrial production (reg. no.
       CZ.02.01.01/00/22\_008/0004590).} }

\usepackage{microtype}

\usepackage{titlesec}
\titlespacing*{\section}{0pt}{*1}{*0.5}
\titlespacing*{\subsection}{0pt}{*0.8}{*0.4}

\begin{document}

\maketitle

\begin{abstract} 
  In this letter, we present a distributed algorithm for flocking in complex
  environments that operates at constant altitude, without explicit
  communication, no a priori information about the environment, and by using
  only on-board sensing and computation capabilities. We provide sufficient
  conditions to guarantee collision avoidance with obstacles and other robots
  without exceeding a desired maximum distance from a predefined set of
  neighbors (flocking or proximity maintenance constraint) during the mission.
  The proposed approach allows to operate in crowded scenarios and to explicitly
  deal with tracking errors and on-board sensing errors. The algorithm was
  verified through simulations with varying number of UAVs and also through
  numerous real-world experiments in a dense forest involving up to four UAVs.
\end{abstract} 


{\small{\textbf{\textit{Video---}\url{https://mrs.fel.cvut.cz/rbl}}}}

{\small{\textbf{\textit{Github---}\url{https://github.com/ctu-mrs/rbl}}}}  


\IEEEpeerreviewmaketitle

\section{Introduction} \label{sec:Introduction} Over the past few decades,
numerous applications have emerged for multi-robot systems, exhibiting their
undisputed benefits in diverse fields. These applications span domains such as
logistics, surveillance, environment monitoring, disaster response, search and
rescue, and agriculture among others, demonstrating their versatility. For many
of these applications, having a swarm of robots able to flock together towards a
desired location can be highly beneficial, due to redundancies and
parallelization. Depending on the environments where robots operate, the
reliability or availability of explicit communication between robots may be
compromised. 

In this work, we propose a Lloyd-based algorithm that relies solely on
on-board sensing. We have synthesized a distributed algorithm in which each
robot safely reaches its desired position while maintaining proximity to its
teammates in cluttered environments. We choose Lloyd algorithm as a base for our
method as it enables seamless multi-robot coordination without requiring robots
to share unobservable or global information, and its safety guarantees are not
dependent on finely tuned parameters.

Our algorithm takes inspiration from the nature; just like flock of birds
or a school of fish, the robots do not need a leader and each robot makes its
own decisions based on relative information, without relying on explicit
communication. We validate our approach through simulations and experiments
in the forest (see Figure~\ref{fig:figure}).
\begin{figure}
     \centering
\includegraphics[width=1.00\columnwidth]{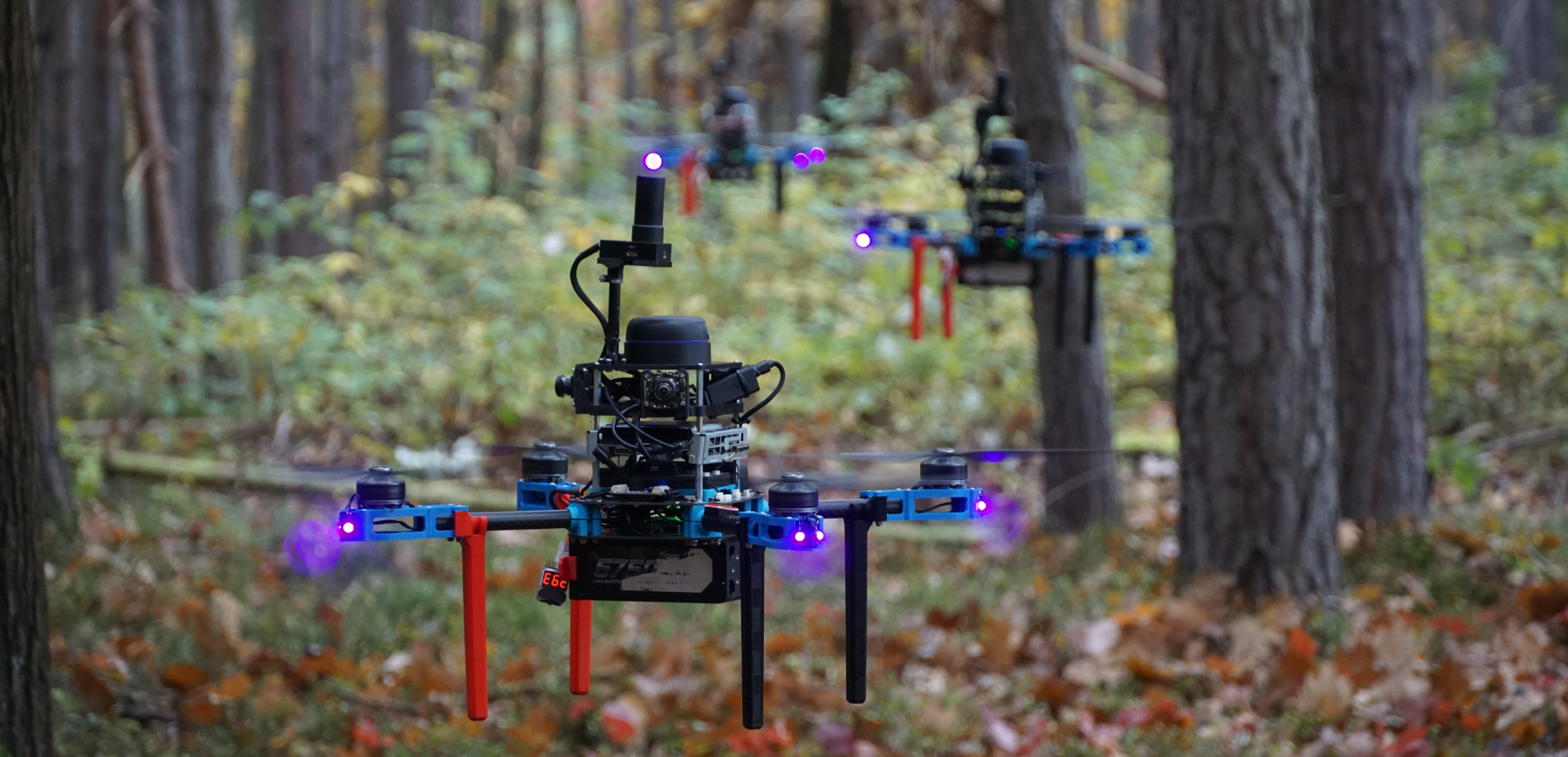}
    \caption{Deployment of the proposed algorithm in a real-world scenario with
    $3$ UAVs navigating in a forest.} \label{fig:figure}
\end{figure}

\subsection{Related works}
\label{sec:Introduction-a}
Prior literature provides a considerable amount of algorithms for distributed
motion coordination of multiple robots, but only a few of them can manage to
work in a cluttered environment by also imposing the flocking constraint. Among
them, a smaller subset is devoted to applications in the field without relying
on explicit communication between the robots. 

To position our work within the state of the art, we consider the
following features.
    \textit{Distributed control}: each robot makes decisions on the basis of
      local information, which is obtained by relying only on
      on-board sensors.
    \textit{Asynchronous control}: each robot makes decisions independently,
    in asynchronous fashion. 
    \textit{Without explicit communication}: the robots do not exchange 
    time-dependent information such as future trajectories, current positions or
    velocities. Moreover, they do not rely on a common reference frame.
   \textit{Safety}:  collision and obstacle avoidance is
    guaranteed at any time.
     \textit{Proximity maintenance}: the distance between robots is guaranteed
     to be below the set threshold on maximum distance.
  \textit{Uncertainties}: the approach is robust to bounded tracking
  errors and bounded measurement uncertainties.
  \textit{Cluttered and highly dense environments}: the algorithm can operate
     in cluttered environments such as a forest.
     \textit{On-board localization}: the algorithm is verified in unknown
    environments relying only on on-board localization. 
     \textit{Experiments}: the algorithm is tested on real robotic platforms
    in unstructured environments. \textit{Open source code}:
    the algorithm is open source.
    \begin{small}
\begin{table*} \centering \renewcommand{\tabcolsep}{0.035cm}
\renewcommand{\arraystretch}{1.0}
  \begin{tabular}{|c|c|c|c|c|c|c|c|c|c|c|c|c|c|c|c|c|c|c|c|c|c|c|c|c|c|c|} \hline
    \bf{Method}  & \cite{boldrer2023rule}&  \cite{boldrer2022unified}&
    \cite{zhou2022swarm}  & \cite{zhou2021ego} & \cite{zhu2024swarm} &
    \cite{ahmad2022pacnav} & \cite{petravcek2020bio}
    &\cite{boldrer2020socially}
    &\cite{zhu2022decentralized} & \cite{ferranti2022distributed} &
    \cite{adajania2023amswarm} & \cite{boldrer2020lloyd}&
    \cite{vasarhelyi2014outdoor} & \cite{vasarhelyi2018optimized} &
    \cite{soares2015distributed} &\cite{wasik2016graph}&
    \cite{mezey2024purely}& \cite{hauert2011reynolds}&
    \cite{quintero2013flocking}& \cite{turpin2012decentralized}&
    \cite{toumieh2024high} & \cite{bartolomei2023fast}& \cite{tordesillas2021mader}
     & \cite{soria2021predictive}&\cite{park2025decentralized} &ours\\ \hline Distributed & \checkmark &\checkmark
    & $\circ$ & $\circ$& $\circ$ &\checkmark & \checkmark & \checkmark & \checkmark &\checkmark  &
    o& \checkmark & \checkmark &\checkmark & \checkmark& \checkmark &\checkmark
    & \checkmark & \checkmark & \checkmark & $\circ$
    &\checkmark&\checkmark&$\times$&\checkmark& \checkmark\\ \hline
    Asynchronous  &  \checkmark &  \checkmark & $\circ$ & $\circ$ & $\circ$ & \checkmark & \checkmark
    & $\times$ & \checkmark  & \checkmark & $\times$  & \checkmark & \checkmark & \checkmark
    &\checkmark &\checkmark& \checkmark&\checkmark&\checkmark&$\times$&$\times$&
    \checkmark&\checkmark&$\times$&\checkmark&\checkmark\\
    \hline w/o Comm.  & $\circ$ & $\times$ & $\times$ & $\times$ & $\times$ & \checkmark & \checkmark & $\times$ & $\circ$ & $\times$ & $\times$
    & $\times$ & $\times$ & $\times$ &\checkmark & $\times$ &\checkmark&
    $\times$ & $\times$ & $\times$ &$\times$&$\times$&$\times$&$\times$&$\times$&\checkmark\\ \hline
    Safety  & \checkmark & \checkmark & \checkmark & \checkmark & \checkmark &
    $\circ$ & $\circ$  & $\circ$ & \checkmark & \checkmark & \checkmark & \checkmark & $\circ$ & $\circ$ & $\circ$
    & $\circ$& $\circ$&$\circ$&
    $\circ$&\checkmark&\checkmark&$\times$&\checkmark&\checkmark&\checkmark&\checkmark \\ \hline Uncertainties  & $\circ$ & $\circ$ & $\circ$ & $\circ$&$\circ$ & $\circ$
    & $\circ$ & $\circ$ & \checkmark & $\circ$ & $\circ$& $\circ$ & $\circ$ &
    $\circ$ & $\circ$ & $\circ$ & $\circ$ &$\circ$
    &$\circ$&$\circ$&$\circ$&$\times$&$\circ$&$\circ$&$\circ$&\checkmark \\
    \hline Prox. Maint.  & $\times$ & \checkmark & \checkmark & $\times$ & $\circ$ & $\circ$ & $\circ$ & $\circ$ & $\times$ & $\times$
    & $\times$ & $\circ$ & $\circ$ & $\circ$ & $\circ$ & $\circ$ & $\circ$
    &$\circ$& $\circ$&\checkmark&$\times$&$\times$&$\times$&$\circ$&$\times$&\checkmark \\ \hline Clutt.
    Envir.  & $\times$ & \checkmark & \checkmark & \checkmark & \checkmark & $\circ$ & $\circ$ &
    \checkmark & \checkmark & $\times$ & \checkmark &\checkmark & $\times$ &\checkmark &$\times$
    &\checkmark
    &$\times$&$\times$&$\times$&$\times$&\checkmark&\checkmark&\checkmark&\checkmark&\checkmark&\checkmark \\ \hline On-board Loc. & $\times$ &$\times$
    &\checkmark & \checkmark & \checkmark &\checkmark & \checkmark &$\times$ &$\times$ &$\times$ &$\times$
    &$\times$ &$\times$ &$\times$ &$\times$ &$\times$ &$\times$ &$\times$
    &$\times$ &$\times$ &$\times$ &$\times$&$\times$ &$\times$&$\times$&\checkmark\\ \hline Experiments  &
    \checkmark & \checkmark & \checkmark & \checkmark&\checkmark & \checkmark &
    \checkmark & $\times$ & \checkmark & \checkmark & \checkmark & $\circ$ & \checkmark
    &\checkmark & \checkmark &\checkmark &\checkmark &\checkmark &\checkmark
    &$\circ$& $\circ$&$\times$&\checkmark&\checkmark&\checkmark&\checkmark \\ \hline Open Source Code & \checkmark & \checkmark & \checkmark
    & \checkmark & \checkmark& \checkmark & \checkmark & $\times$ & $\times$ & $\times$ & \checkmark &
    $\times$ & $\times$ & $\times$ & $\times$ & $\times$ & \checkmark & $\times$
    & $\times$&$\times$&\checkmark&\checkmark&\checkmark&\checkmark&\checkmark&\checkmark \\ \hline
  \end{tabular}
\vspace{10pt} 
  \caption{List of state-of-the-art algorithms. We indicate with
``$\times$'' if the algorithm does not implement the feature as described in
  Section~\ref{sec:Introduction-a}, ``\checkmark'' if it
does, and  ``o'' if it does only partially.}
\label{tab:sota} \end{table*}  \end{small}

In Table~\ref{tab:sota}, we report an overview of the most closely related
algorithms in the literature. We mark with ``$\times$'' if the algorithm does
not implement the feature, ``\checkmark'' if it does, and finally, ``o'' if it
does but only partially.  As shown in Table~\ref{tab:sota}, our algorithm is
unmatched in fulfilling all the specified features of interest. In fact,
with respect to~\cite{boldrer2023rule}, we included the proximity
maintenance constraints, the ability to deal with cluttered environments,
tracking errors and uncertainties, without relying on explicit communication,
nor an external localization system. Despite~\cite{zhou2021ego,zhou2022swarm}
proposing valid solutions to the problem, their main limitations are the
requirement of a trajectory broadcasting network and the absence of account of proximity maintenance, tracking errors and uncertainties. Similarly,
in~\cite{zhu2024swarm} the robots need to broadcast information such as poses,
velocities, poses covariance, and global extrinsic transformations. On the other
hand,~\cite{ahmad2022pacnav,petravcek2020bio}, do not rely on an explicit
communication network, but they exhibit degraded performance in crowded
scenarios, lack of safety guarantees, proximity maintenance, and robustness to
tracking errors and uncertainties. In~\cite{zhu2022decentralized}, the authors
present a distributed solution, which in theory does not rely on explicit
communication and accounts for uncertainties. However, in the experiments they
rely on motion capture measurements adding limited noise errors, which is
reasonable only in structured environment (i.e., laboratory conditions). With
respect to~\cite{boldrer2022unified,boldrer2021graph}, where the authors
introduced the flocking behavior in the Lloyd-based framework, we do not impose
the line-of-sight constraints, i.e., two robots are considered in the proximity
if they are close enough. This allows to increase the robots mobility and also
to not rely on explicit communication. In Table~~\ref{tab:sota}, we report
additional relevant state-of-the-art algorithms. Notably, all of them share
limitations related to the use of an explicit communication network and/or to
not account for safety, proximity maintenance and the uncertainties.
Conversely, in our work, we provide a method which implements all the
fore-mentioned features as well as a comprehensive rule for the selection of the
algorithm's parameters based on measurable quantities (e.g., tracking error, or
covariance of state estimates), which is often an unclear or empirical process
in the state-of-the-art algorithms.

\subsection{Contribution and paper organization} Our contribution is threefold.
First, we propose a novel algorithm for flocking control in dense cluttered
environments. The key novelties lie in our use of an adaptive control strategy
to compensate for tracking errors and measurement uncertainties, combined with
Convex Weighted Voronoi Diagrams (CWVD) to enhance the performance in a dense
forest scenarios. Second, we provide sufficient conditions for safety and
proximity maintenance in cluttered environments and in the presence of flocking
constraints. Finally, we validate the algorithm in the forest, relying only on
on-board sensors, i.e., Inertial Measurement Unit (IMU), 2-D Lidar, altitude
sensor, and UVDAR system~\cite{walter2019uvdar}, without relying neither on an
explicit communication network nor GNSS signals nor having a common
reference frame.


\section{Problem description}
\label{sec:problem description}
Let us denote the positions of the robots with $p= [p_1,\dots, p_N]^\top$, where $p_i
= [x_i,y_i]^\top$ and the obstacle positions by $o = [o_1,\dots,o_{N_o}]^\top$,
$o_k = [x_k,y_k]^\top$. The mission space is represented by $\mathcal{Q} \in
\mathbb{R}^2$, the
weighting function that assesses the weights of the points $q \in \mathcal{Q}$ is
$\varphi_i(q) : \mathcal{Q} \rightarrow \mathbb{R}_+$, while $\delta_i$ and
$\delta_k^o$ are the occupancy radius of the $i$--th robot and the $k$--th
obstacle respectively. The problem that we aim to solve reads as follow:

\begin{problem}
    Let $N$ be the number of robots in the swarm, let $\mathcal{N}_i$ be the set
    of neighbors of the robot $i$, and $\bar{\mathcal{N}}_i \subseteq
    \mathcal{N}_i$ the set of neighbors that has to remain in the proximity of
    the robot $i$ for
    the entire mission. We want to synthesize a distributed control law, such
    that the swarm maintains the desired proximity constraints, which is
    characterized by a set of maximum distances $\Gamma=\{\Gamma_{ij} \mid
    i \in {1,\dots,N}, j \in \bar{\mathcal{N}}_i\}$. Each robot in the swarm
    has to reach its goal region avoiding collisions with other robots and
    obstacles.
\end{problem}

\section{Proposed Approach}
\label{sec:approach}

The presented approach is grounded in the Lloyd algorithm~\cite{lloyd1982least}.
Let us consider the following cost function
\begin{equation}\label{eq:jcov}
  J_{\operatorname{cov}}({p}, \mathcal{V})=\sum_{i \in \mathcal{N}_i }\int_{\mathcal{V}_i}\left\|q-p_i\right\|^2 \varphi_i(q) d q,
\end{equation}
where $\mathcal{V}= \{ \mathcal{V}_1, \dots, \mathcal{V}_N \}$,
$    \mathcal{V}_i= \mathcal{V}_i^{p} \cap \mathcal{V}_i^{o} 
$
with $\mathcal{V}_i^p$ being the $i$--th Convex Weighted Voronoi Diagram (CWVD)
introduced in~\cite{boldrer2020lloyd}, which  accounts only for neighboring
robots' positions, while $\mathcal{V}_i^o$ is the $i$--th CWVD accounting only
for neighboring obstacles' positions:

\begin{equation}\label{eq:cwvd}
  \small
  \begin{split}
  \mathcal{V}_i^p = \{q \in \mathcal{Q} \mid 
       w_i^x \cos \alpha_{ij} + w_i^y \sin \alpha_{ij} < \frac{1}{\epsilon_p}
       &\left\| p_i - p_j \right\|, \\
      & \forall j \in \mathcal{N}_i \}.
  \end{split}
\end{equation}
\begin{equation}
  \small
  \begin{split}
    \mathcal{V}_i^o=\{q \in \mathcal{Q} \mid w_i^x \cos \alpha_{ij} + w_i^y
    \sin \alpha_{ij} < \frac{1}{\epsilon_o}&\left\|p_i -o_j\right\| ,\\ &\forall j \in
    \mathcal{O}_i \},
  \end{split}
\end{equation}
\noindent where $w_i = [w_i^x,w_i^y]^\top = q - p_i$ and $\alpha_{ij} =
\arctan\left(\frac{y_j-y_i}{x_j-x_i}\right)$, while $\epsilon_{p(o)} \in
[1,2]$ regulates the cautiousness in the
robot motion (i.e., $\epsilon_{p(o)} = 1$ more aggressive behavior,
$\epsilon_{p(o)} = 2$ more conservative behavior, notice that the latter
coincides with the classic Voronoi diagram).
\begin{figure}[t]
 \centering
 \setlength{\tabcolsep}{0.05em}
 \begin{tabular}{cc}
   \includegraphics[width=0.4\linewidth]{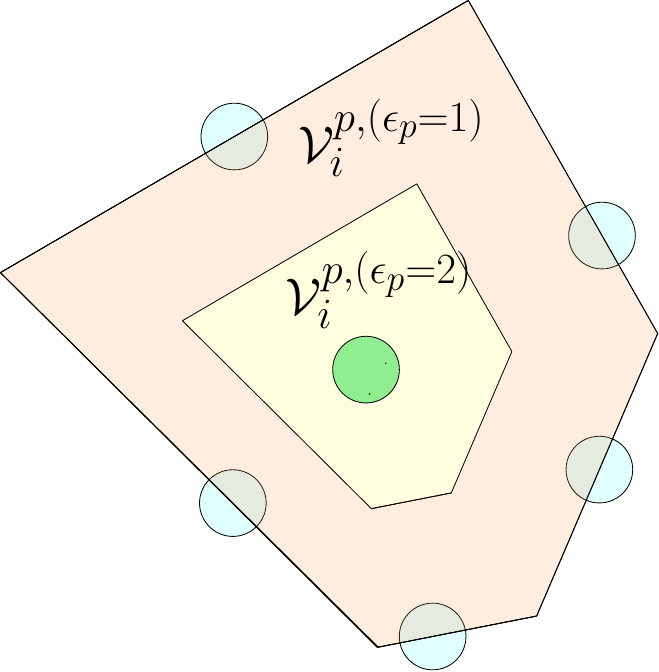} &
  \includegraphics[width=0.4\linewidth]{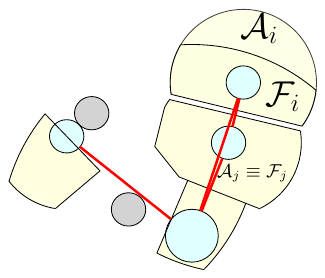}\\
   (a) & (b)
   \end{tabular}
  \caption{(a) Representation of the Convex Weighted Voronoi Diagram (CWVD)
  $\mathcal{V}_i^p$ in~\eqref{eq:cwvd} with $\epsilon_p=2$ and $\epsilon_p=1$.
  Where the green circle is the i--th robot and the cyan circles represent the
  neighboring robots. (b) Example of $\mathcal{F}_i$ and $\mathcal{A}_i$ cell
  geometry with $\epsilon_p =\epsilon_o = 2$. Red solid lines indicate the
  proximity maintenance constraints. Obstacles are represented as grey
  circles, while the robots are depicted in cyan. }
  \label{fig:vorocell}
\end{figure}
In Figure~\ref{fig:vorocell}a, we depict an example of the CWVD~\eqref{eq:cwvd} for
the case where $\epsilon_p=2$ and $\epsilon_p=1$.

 By following the gradient descent of the cost function~\eqref{eq:jcov}, we obtain
the following proportional control law
\begin{equation}\label{eq:lb}
  \dot{p}_i\left(\mathcal{V}_i\right)=-k_{p,i}\left(p_i-c_{\mathcal{V}_i}\right),
\end{equation}
where $k_{p,i}>0$ is a positive real number, and
\begin{equation}\label{eq:centroiddef}
  c_{\mathcal{V}_i}=\frac{\int_{\mathcal{V}_i} q \varphi_i(q) d q}{\int_{\mathcal{V}_i} \varphi_i(q) d q}
\end{equation}
is the centroid position evaluated over the $\mathcal{V}_i$ cell. By
imposing~\eqref{eq:lb}, assuming a time-invariant $\varphi_i(q)$, each robot
$i \in 1, \dots, N$, converges asymptotically to the centroid position
$c_{\mathcal{V}_i}$~\cite{cortes2004coverage}. In the following, we show how we can
attain a safe and effective algorithm for multi-robot swarming by introducing
suitable modifications to the geometry of cell $\mathcal{V}_i$ and crafting an
appropriate weighting function $\varphi_i(q)$
\subsection{Reshaping the cell geometry $\mathcal{V}_i$} 
\noindent The reshape of the cell geometry is both
necessary and sufficient to ensure safety guarantees. Hence, we define the cell
$\mathcal{A}_i = \tilde{\mathcal{V}}_i^p \cap \tilde{\mathcal{V}}_i^o  \cap
\mathcal{S}_i$, where
\begin{equation}\label{eq:Vip}
  \tilde{\mathcal{V}}_i^p=\left\{\begin{array}{l}
    \mathcal{V}_i^p  \text { if } \frac{1}{2}
    \left\|p_i-p_j\right\| > \Delta_{i j} ,\forall j \in \mathcal{N}_i, \\
    \left\{q \in \mathcal{Q} \mid\left\|q-p_i\right\|
    \leq\left\|q-\tilde{p}_j\right\|, \forall j \in \mathcal{N}_i
\right\} \text { otherwise, }  \end{array}\right.
\end{equation}
where $\Delta_{i j}=\delta_j+\delta_i$,  and
$\tilde{p}_j=p_j+2\left(\Delta_{i j}-\frac{\left\|p_i-p_j\right\|}{2}\right)
\frac{p_i-p_j}{\left\|p_i-p_j\right\|}$, takes into consideration the physical
dimension of the robots. To account for the obstacles,
\begin{equation}\label{eq:Vio}
  \tilde{\mathcal{V}}_i^o=\left\{\begin{array}{l}
    \mathcal{V}_i^o  \text { if } \frac{1}{2}
    \left\|p_i-p_j\right\| > \Delta^o_{i j} ,\forall j \in \mathcal{O}_i,  \\
    \left\{q \in \mathcal{Q} \mid\left\|q-p_i\right\|
    \leq\left\|q-\tilde{o}_j\right\|, \forall j \in \mathcal{O}_i \right\} \text { otherwise, }
  \end{array}\right.\
\end{equation}
where $\Delta^o_{i j}=\delta_i+\delta^o_j$  and
$\tilde{o}_j=o_j+2\left(\Delta^o_{i j}-\frac{\left\|p_i-o_j\right\|}{2}\right)
\frac{p_i-o_j}{\left\|p_i-o_j\right\|}$, takes into consideration the 
size of the robot and the obstacles, while
\begin{equation} \mathcal{S}_i=\left\{q  \in \mathcal{Q}
\mid\left\|q-p_i\right\| \leq r_{s, i}\right\}, 
\end{equation} where $r_{s,i}$
is half of the sensing radius $R_s$ of the robot $i$.

\begin{theorem}[Safety]\label{th:1}
  By imposing the control $\dot{p}_i(\mathcal{A}_i)$ in~\eqref{eq:lb}, with
$\mathcal{A}_i = \tilde{\mathcal{V}}_i^p \cap \tilde{\mathcal{V}}_i^o  \cap
\mathcal{S}_i$, collision avoidance is guaranteed at every instant of time.
\end{theorem}
\begin{proof}
  The proof comes from the definition of convex set. Due to the fact that the
  set $\mathcal{A}_i$ is always convex, and by definition $p_i,c_{\mathcal{A}_i}
  \in \mathcal{A}_i$, it follows that since the straight path from the robot
  position to the centroid position is a safe path, then applying $\dot{p}_i(\mathcal{A}_i)$
  in~\eqref{eq:lb} preserves safety. 
  Due to the robots' size, the condition that the straight path from the
  robot to its centroid position is a safe path can be violated only if the
  $i$-th robot is at a distance greater than $2\Delta_{ij}$ ($2\Delta^o_{ij}$)
  from the $j$-th robot ($j$-th obstacle), according
  to~\eqref{eq:Vip},~\eqref{eq:Vio}. Nevertheless, to experience collision the
  distance between the $i$-th robot and the $j$-th robot ($j$-th obstacle) needs
  to be less than $2\Delta_{ij}$ ($2\Delta^o_{ij})$, i.e., equal to
  $\Delta_{ij}$ ($\Delta^o_{ij}$), hence safety is preserved.
\end{proof}

We further apply an additional reshaping of the cell geometry in order to
maintain the desired proximity of the agents, i.e., to achieve flocking
behavior. By imposing our final cell
geometry $\mathcal{F}_i =  \mathcal{A}_i \cap \mathcal{M}_i $, where
\begin{equation}\label{eq:conn}
  \mathcal{M}_i = \{q \in \mathcal{Q} \mid \| q-p_j \| \leq \Gamma_{ij}, \forall j \in \bar{\mathcal{N}}_i \},
\end{equation}
we obtain proximity maintenance while preserving safety. In
Figure~\ref{fig:vorocell}b, we depict an example of cell geometry $\mathcal{F}_i$.

\begin{theorem}[Proximity maintenance or flocking constraint]
  By imposing the control $\dot{p}_i(\mathcal{F}_i)$ in~\eqref{eq:lb}, both
  safety and proximity maintenance are guaranteed at every instant of time.
\end{theorem}
\begin{proof}
  Given that $\mathcal{F}_i \subseteq \mathcal{A}_i$, safety comes from
  Theorem~\ref{th:1}. Since $\mathcal{M}_i$ is convex by definition,
  $\mathcal{F}_i$ is also convex, and thus the centroid $c_{\mathcal{F}_i} \in
  \mathcal{F}_i$. It  implies that by imposing $\dot{p}_i(\mathcal{F}_i)$
  in~\eqref{eq:lb}, $p_i \in \mathcal{F}_i$ at any time. Hence, by the
  definition of the $\mathcal{F}_i$ set, proximity maintenance between $i$ and
  $j$ is always preserved.
\end{proof}
\subsection{The weighting function $\varphi_i(q)$}
Although ensuring safety and proximity maintenance only involves properly
shaping the cell geometry, convergence to the goal and enhancing navigation
performance largely hinges on how the function $\varphi_i(q)$ is defined.
As originally proposed in~\cite{boldrer2023rule}, we adopt a
Laplacian-shaped function centered at $\bar{p}_i$, which represents a point in
the mission space that ultimately needs to converge to the goal location. We
design the $\varphi_i(q)$ function as follows:
\begin{equation}
  \label{eq:phi}
  \varphi_i(q) =
  \text{exp}{ \left( -\frac{\|q-\bar{p}_i\|}{\beta_i}\right)},
\end{equation}
where the spreading factor $\beta_i \in \mathbb{R}$ is time-variant. It follows
the dynamics \begin{equation}
  \label{eq:rho}
  \dot{\beta}_i(\mathcal{A}_i) = \begin{cases} -k_{\beta}\beta_i \,\,\,
    &\text{if}\,\, \|c_{\mathcal{A}_i}-p_i\| <d_{1} \, \land \,\\
    &\|c_{\mathcal{A}_i} -c_{\mathcal{S}_i}\|> d_{2}, \\
    -k_{\beta}(\beta_i-\beta^D_i) &\text{otherwise,} \end{cases}
\end{equation}
and

\begin{equation}
  \label{eq:wp}
  \begin{split}
    \dot{\bar{p}}_i &= \begin{cases}
      -k_e(\bar{p}_i -R(\pi/2-\varepsilon)e_i) \,\,\, &\text{if}\,\,
      \|c_{\mathcal{A}_i}-p_i\| < d_{3} \, \land \,\\ &\| c_{\mathcal{A}_i} -
      c_{\mathcal{S}_i}\|> d_{4}, \\ -k_e(\bar{p}_i-e_i)
      &\text{otherwise,}
    \end{cases} \\
    \bar{p}_i &= e_i  \,\,\, \text{if} \,\,\, \| p_i - \bar{c}_{\mathcal{A}_i}\| > \| p_i - c_{\mathcal{A}_i}\| \, \land \, \bar{p}_i = R(\pi/2-\varepsilon)e_i. \\
  \end{split}
\end{equation}
The parameters $\beta_i^D, d_1, d_2, d_3, d_4,k_{\beta}, k_e$ are
positive scalars, $R(\theta)$ is the rotation matrix, $e_i$ is the goal position
for the $i$--th robot, $\varepsilon$ is a small number, which allows to avoid
rotations equal to $\pi/2$, while $\bar{{c}}_{\mathcal{A}_i}$ is the centroid
position computed over the cell $\mathcal{A}_i$ with $\bar{p}_i \equiv e_i$.
Equations~\eqref{eq:rho} and~\eqref{eq:wp} specify two adaptive rules.
Intuitively,~\eqref{eq:rho} regulates the greediness of the robot $i$ in
reaching its goal region, while~\eqref{eq:wp} introduces a right (or left)
hand-side rule to enforce a common behavioral convention across the robots and
hence enhance inter-robot coordination.

\subsection{Dealing with dynamics and uncertainties}\label{sec:uncertainties}
The robot deployment requires to consider a set of additional issues.
First, more complex dynamics of the robots must be applied as $\dot{p}_i = u_i$
is not directly applicable. As a consequence, tracking error, which is caused by
model mismatch and external disturbances, generating a difference between the
desired and the actual velocity, has to be taken into account. Finally, we have
to consider measurement uncertainties arising from limitations of the perception
system providing an estimate of the positions of other robots and obstacles with
a certain error.

To account for the dynamic constraints of the robots, we synthesize the control
inputs through an MPC formulation (similarly
to~\cite{zhu2022decentralized,boldrer2023rule}). Hence, each robot $i$ solves the following optimization
problem (for the sake of clarity we discard the $i$ indexing),
\begin{subequations} \label{eq: MPC problem_formulation}
  \begin{align}\label{eq:cost_generic} \minimize_{{x}_k,{u}_k}& \sum^{N_t}_{k=0}
    J({x}_k,{u}_k, c_{\mathcal{F}}),\\ \label{eq:MPC
    dyn_constraints_generic}\textrm{s.t.: } & {x}_{k} = f({x}_{k-1},{u}_{k-1}),
    \,\,\, \forall k \in (1,N_t) ,\\ \label{eq:MPC bounds_generic}&\,{x}_k\in
    \mathcal X, \, {u}_k \in \mathcal U, \,\,\, \forall k \in (0,N_t),\\
   \label{eq:init_generic} &{x}_0 = {x}_{\textrm{init}},
  \end{align} \end{subequations} where $k$ indicates the time index, $N_t$
  indicates the horizon length, $J$ represents the cost function,
  ${x}_{k}=[p_k^\top,\dot{p}_k^\top,\ddot{p}_k^\top]^\top$ and
  ${u}_{k}$ are the state and control inputs respectively,
  $f({x}_{k-1},{u}_{k-1})$ indicates the dynamics of the robot, ${x}_k\in
  \mathcal X, \, {u}_k\in \mathcal U$ are the state and input constraints
  respectively, and ${x}_0 = {x}_{\textrm{init}}$ is the initial
  condition.

Both the tracking errors and measurement uncertainties can be
accounted for by means of an adaptive strategy. Specifically, it is 
achieved by adapting the value of $\beta_i$ in~\eqref{eq:rho} in an
online fashion. An alternative approach would be reshaping the
$\mathcal{F}_i$ cell, similarly to~\cite{zhu2022decentralized}. However, this
method cannot be applied in our case, since, by following this approach, the
robot may fall out of its cell due to tracking error, and once it happens, both
the cell and the centroid position are no longer defined. Conversely, the value
of $\beta_i$ defines the aggressiveness of the $i$--th robot. A smaller value
corresponds to a centroid position closer to the goal (and hence closer to the
boundaries of the $\mathcal{F}_i$ cell, namely $\partial \mathcal{F}_i$), while
larger values of $\beta_i$ push the centroid position as far as possible from
the boundaries of the cell $\partial \mathcal{F}_i$. Thus, managing errors and
uncertainties is achieved by selecting a value of $\beta_i$ such that the
centroid position stays at a controlled distance from the boundary of
the cell. This distance must be chosen in accordance with the maximum
positional tracking error and the maximum error in the neighbors' position
estimates. More formally, it can be done by imposing a lower bound to $\beta_i$
in \eqref{eq:rho} equal to
  \begin{equation}\label{eq:beta_min}
\beta^{\min}_i = \argmin_{\beta_i} \left(\|c_{\mathcal{F}_i} - p_{\partial
  \mathcal{F}_i} \| - d_u\right)^2,
\end{equation} 
  where  $p_{\partial \mathcal{F}_i}$ is the position of the closest point to
  $c_{\mathcal{F}_i}$ on the boundary of $\mathcal{F}_i$, while $d_u$ is
  a threshold value that has to be
tuned on the basis of the tracking errors and measurement uncertainties. For
the sake of clarity, we depicted three scenarios in Figure~\ref{fig:errors}.
Figure~\ref{fig:errors}a considers the tracking error; in this case the
value of $c_{\mathcal{F}_i}$ is controlled to stay at least at a distance
$d^t_u$ from the boundary $\partial \mathcal{F}_i$. We indicate with
$\mathcal{Z}_i$ the region set where the $i$--th robot is allowed to go in the
next time step. Notice that the dimension of the region $\mathcal{Z}_i$ can be
controlled by $d_u^t$. To be conservative, the value of $d_u^t$ should be
determined empirically as the maximum tracking error, and it can be set as a
function of the velocity. In Figure~\ref{fig:errors}b, we show the case 
where we have measurement uncertainties. We illustrate both the $i$--th and the
$j$--th robots' positions and the ellipse of uncertainty $\mathcal{E}_{i,j}$.
In this case, $\mathcal{K}_i$ indicates the region where the centroid
  $c_{\mathcal{F}_i}$ has to be constrained to preserve safety. Let us consider
  the case where $\epsilon_{p(o)}=1$. By relying on
  geometric reasoning (refer to Figure~\ref{fig:errors}b), depending on
  the inter-robot distance measurement uncertainty, we can select the
  parameter $d_u^m$ as \begin{equation}
    d_u^m \geq k \sqrt{\lambda^{\max}}+\delta_j+\delta_i,
    \label{eq:d_inequality}
  \end{equation}
where $\lambda^{\max}$ is the maximum eigenvalue of the covariance matrix
associated with the distance measurement and $k$ is a scale factor that
defines the confidence ellipse.

From a practical point of view, the value of $d_u^m$ may be set empirically to a
constant that satisfies inequality (\ref{eq:d_inequality}) in the worst expected
scenario.
In real-world equipments the uncertainties associated with the measurements are
far from being constant values. Hence, the user can set $d_u^m$ dynamically,
based on the current measurement uncertainties, which are reported by the
equipment together with the measurement means. One method to implement this, is
to use the \emph{restraining} method from~\cite{walter2023distributed}. This
technique offsets the set-points of the robots from an ideal target location
according to the probability distributions of relative pose observations.

On the other hand, to preserve proximity maintenance, a sufficient
condition is that $d_u^m \geq k \sqrt{\lambda^{\max}}$. Finally, the
tracking error and measurement uncertainties are combined as $d_u =
d_u^m+d_u^t$ in Figure~\ref{fig:errors}c.

Hence, by imposing the control $\dot{p}_i(\mathcal{F}_i)$ in~\eqref{eq:lb}, with
$\mathcal{F}_i = \tilde{\mathcal{V}}_i^p \cap \tilde{\mathcal{V}}_i^o  \cap
\mathcal{S}_i \cap \mathcal{M}_i$ and assuming a maximum uncertainty in the
estimation of the other robots' distances equal to $k\sqrt{\lambda^{\max}}$, by
selecting $d_u^m \geq k\sqrt{\lambda^{\max}}$, and assuming a bounded maximum
tracking error less or equal than $d^t_u$, both safety and proximity maintenance
are preserved by selecting $\beta_i^{\min}$ according to~\eqref{eq:beta_min} for
all $i=1,\ldots,N$.
Specifically, let us choose $d_u^m$ such that it satisfies~\eqref{eq:d_inequality}
and $d_u=d^t_u+d_u^m$.  By enforcing a lower bound
$\beta_i^{\min}$~\eqref{eq:beta_min}, we set a minimum distance $d_u$ between
the centroid position $c_{\mathcal{F}_i}$ and the boundary of the cell $\partial
\mathcal{F}_i$. Due to the definition of $d_u$, this margin explicitly accounts
for both estimation uncertainty and tracking error, thereby ensuring safety and
proximity maintenance for all admissible realizations of the neighbors' positions
within the confidence ellipse. More intuitively, the centroid is
constrained to remain within a region that preserves safety and proximity
maintenance by accounting for tracking error and all possible neighbors'
positions consistent with the uncertainty bounds. If there exists no point $q
\in \mathcal{F}_i$ such that $\text{dist}(q,\partial \mathcal{F}_i) \geq d_u$,
according to~\eqref{eq:beta_min}, the centroid position $c_{\mathcal{F}_i}$, is
the point that maximizes the distance from $\partial \mathcal{F}_i$. Under this
condition, strict guarantees of safety and proximity maintenance cannot be
ensured for all possible uncertainty realizations. Nevertheless, in this case
the robot adopts the most conservative behavior, placing the centroid at the
location that maximizes boundary clearance. The guarantees provided by the
proposed conditions are therefore sufficient and hold whenever a feasible margin
exists; they rely on worst-case realizations within the assumed uncertainty
bounds and bounded tracking errors.

\begin{figure}
  \centering
  \includegraphics[width=.70\columnwidth]{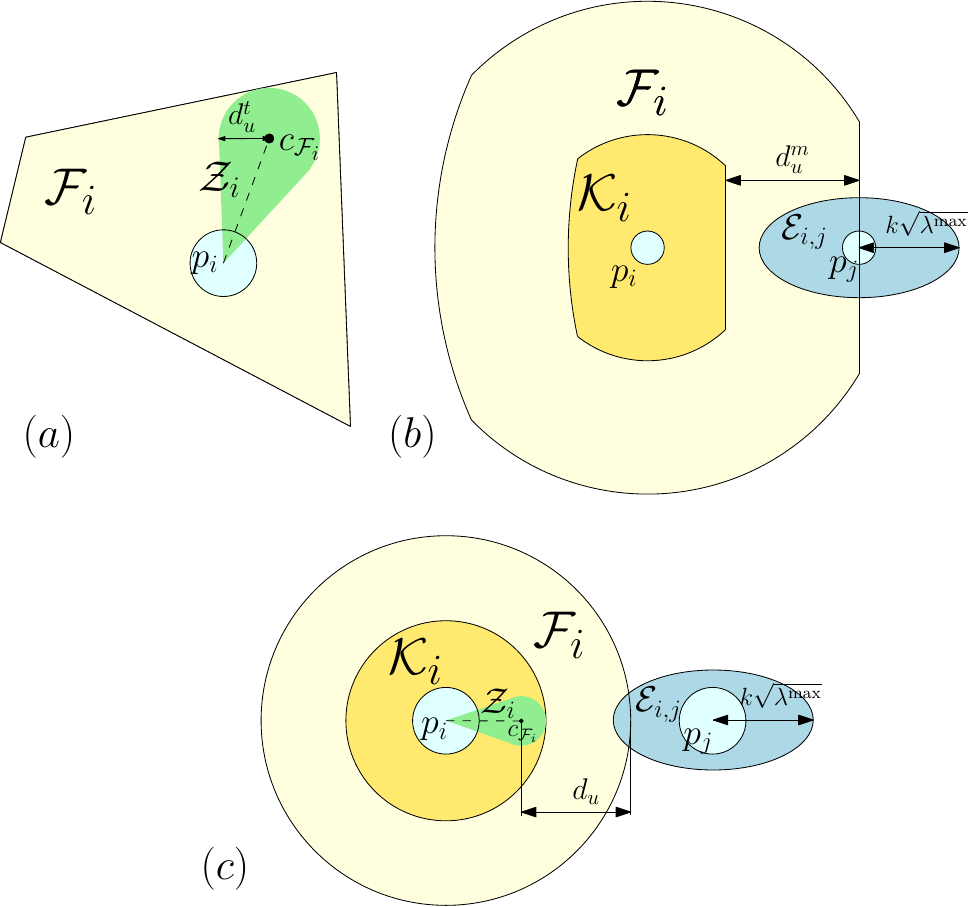}
  \caption{Dealing with uncertainties: (a) tracking error case, (b) measurement
  uncertainty case, (c) combined tracking error and measurement uncertainty. The
  robots are indicated with the cyan circles. We represent in light yellow the
  $\mathcal{F}_i$ set, in dark yellow the set $\mathcal{K}_i$, in green the set
  $\mathcal{Z}_i$, in blue the ellipse of uncertainty $\mathcal{E}_{i,j}$. }
  \label{fig:errors}
\end{figure}

\begin{remark} [Tracking error and measurement uncertainties invariance
  property] Since we are adopting the same strategy to account for the tracking
  errors and measurements uncertainty, we get the property of invariance. In
  practice, it does not matter how  $d_u^m$ and $d_u^t$ are chosen, what
  matters is the sum $d_u$, hence in the case of underestimation of one of the
two terms, it is compensated by the other. \end{remark}

\section{Simulation results} \label{sec:Simulation results}
In this section, we first show through simulations the scalability,
repeatability and robustness of our approach in two different scenarios, with and
without accounting for uncertainties and tracking errors.
Then we report a comparison with PACNav~\cite{ahmad2022pacnav}, the algorithm
does not rely on communication network, and shares similar features with
the proposed algorithm (refer to Table~\ref{tab:sota}). Notice that comparing
with algorithm such as~\cite{zhou2021ego,zhou2022swarm,zhu2024swarm} is out of
scope of this paper. In most of the scenarios, our method does not perform
better than algorithms that can share information, such as future trajectories,
among the robots. Despite~\cite{ahmad2022pacnav} performs well in relatively
simple environments and with a limited number of robots, it lacks reliability in
more complex scenarios. In addition, crucial requirements such as safety and
proximity maintenance depend strongly on fine parameter tuning. This is a common
issue in the state-of-the-art algorithms, which is overcome in our solution,
where the parameter selection affects only the performance of the mission,
while, safety and proximity maintenance are met if they are selected according
to the theoretical findings. All the simulations are performed by relying on the
MRS simulator~\cite{baca2021mrs}, which account for full rigid-body UAV
dynamics.

In the first simulations, we consider different set of parameters
$d_u=[0,1.0]$~(m) in two different scenarios, with $N=[9,16]$, $\delta_i =
0.2$~(m), $\epsilon =1.0$, and randomly generated circular obstacles with
$\delta_o =0.15$~(m) (see Figure~\ref{fig:sim}). We select the parameters as
follows for all the simulations $r_{s,i} = 5.0$~(m), $d_1=d_2= d_3=d_4=1$~(m),
$\beta_i^D=0.15$, $\Gamma_{ij} = 10$~(m), while the adjacency matrices defining
proximity maintenance constraints are defined as grid matrices, as depicted in
Figure~\ref{fig:sim} (red lines). We simulated UVDAR noise as a bounded radial
random measurement error with magnitude not exceeding $0.8$~(m). In
Table~\ref{tab:simulation}, we report the quantitative results in terms of
average time to reach the goal and success rate over a total of $40$ runs. As
expected, setting $d_u=0$~(m) leads to better performance in terms of time to
reach the goal; however, it does not ensure collision avoidance between the
robots. In contrast, selecting $d_u$ according to the theoretical findings
results in safe behavior and preserves the proximity maintenance constraints.
Figure~\ref{fig:sim} illustrates the two scenarios with $N=[9, 16]$ robots, and
shows the trajectories generated for one instance with $d_u=1.0$~(m).
Figure~\ref{fig:dist_sim} illustrates the distance between each robot in the two
scenarios. The colored lines indicate whether the proximity maintenance
constraint is active.
To further validate the robustness of the proposed approach, we additionally
performed 200 simulations across four different parameter sets in the same
scenario depicted in Figure~\ref{fig:sim}a, consistently achieving the
goal preserving safety and the proximity maintenance constraints.

\begin{figure}[t]
 \centering
 \setlength{\tabcolsep}{0.05em}
 \begin{tabular}{cc}
   \includegraphics[width=0.45\linewidth]{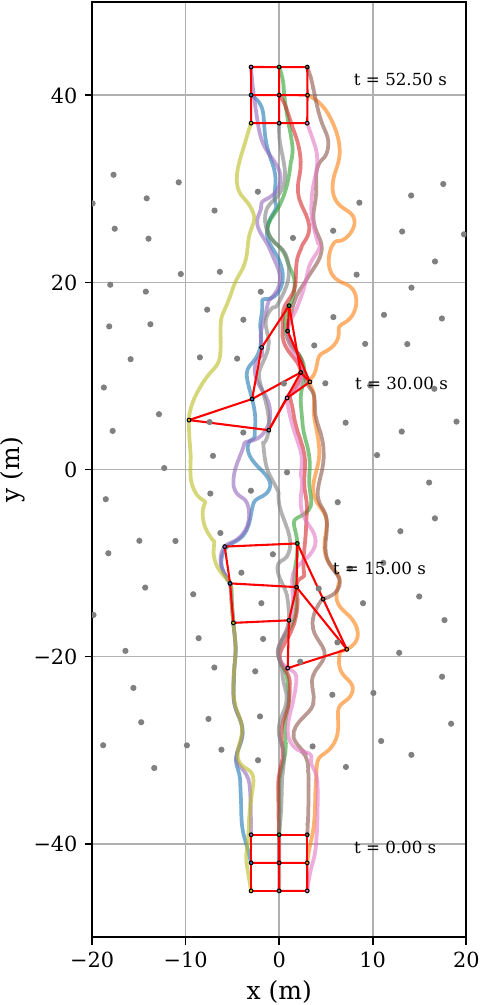} &
  \includegraphics[width=0.45\linewidth]{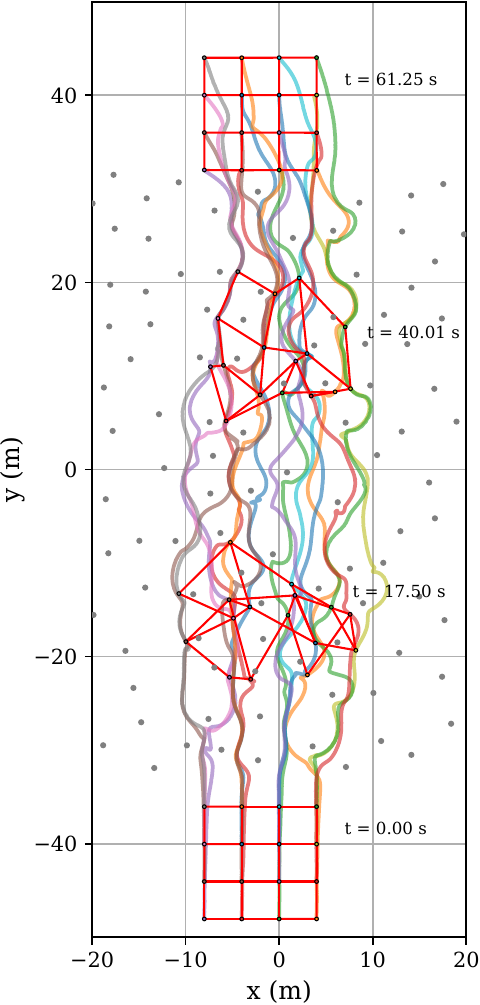}\\
   (a) & (b)
   \end{tabular}
  \caption{Simulation results with $N=9$ and $N=16$ robots, with
  $d_u=1.0$~(m). The images represent example scenarios used for generation of
  results presented in Table~\ref{tab:simulation}. Colored circles represent the
  robots' positions at different time $t$. Red lines represent the proximity
  maintenance constraints between the robots. Black dots are the obstacles in
  the scene.} 
  \label{fig:sim}
\end{figure}

\begin{figure}[t]
 \centering
 \setlength{\tabcolsep}{0.05em}
 \begin{tabular}{cc}
   \includegraphics[width=1.0\linewidth]{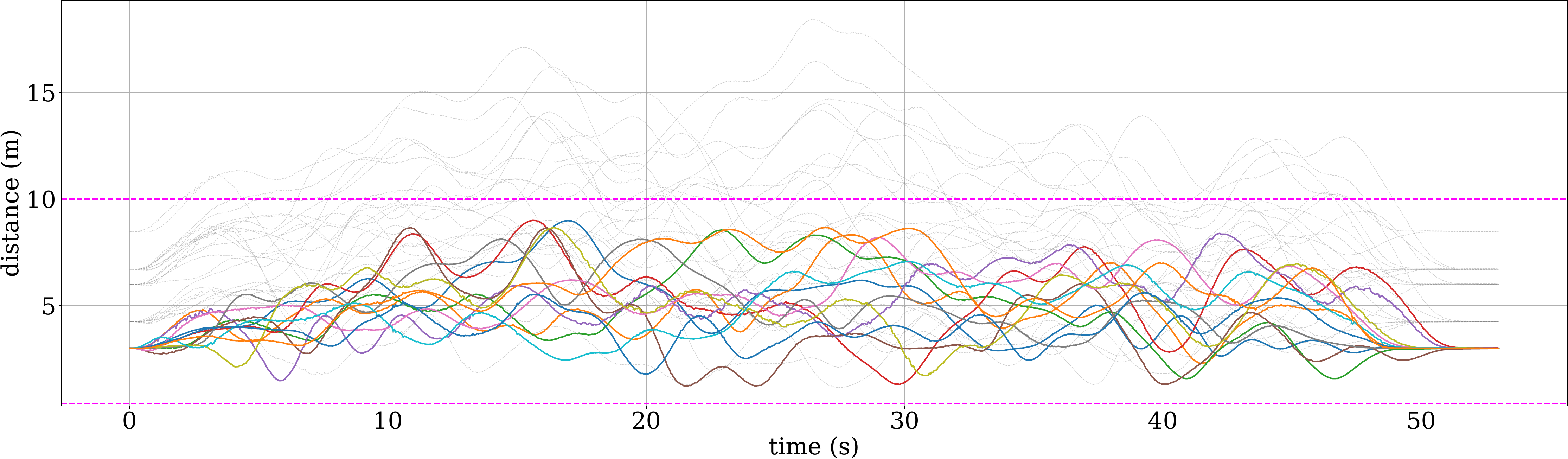}\\ 
  \includegraphics[width=1.0\linewidth]{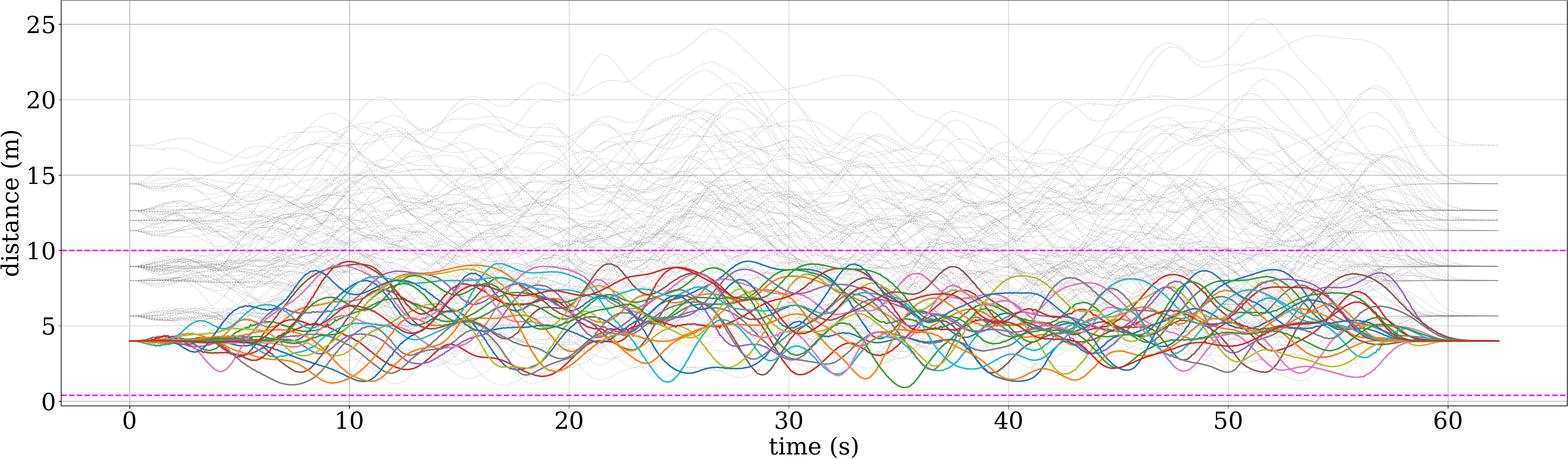}\\
   \end{tabular}
  \caption{Distances between UAVs in the simulations in
  Figure~\ref{fig:sim}a and~\ref{fig:sim}b. The flocking and safety constraints
  are in magenta. Colored lines indicate the presence of an active flocking
  constraint between the two robots.} 
  \label{fig:dist_sim}
\end{figure}

\begin{table}
    \centering
    \renewcommand{\tabcolsep}{0.45cm} 
    \begin{tabular}{|c|c|c|c|c|}
    \hline
       \multicolumn{1}{|c|}{} &\multicolumn{2}{c|}{\footnotesize$d_u =0~(m)$} & \multicolumn{2}{c|}{$d_u =1.0~(m)$} \\ \hline
      N & avg. time (s) & SR & avg. time (s) & SR \\ \hline
      9 & 32.71 & 0.50  & 54.91 & 1.00 \\ \hline
      16 & 34.94 & 0.40  & 64.53 & 1.00 \\ \hline
    \end{tabular}
    \vspace{10pt} 
    \caption{The simulation results obtained in terms of time to reach the
    goal and success rate (SR) for $10$ instances in $2$ different scenarios with
    $N=[9, 16]$. We report the results with different parameter selection i.e.,
    $d_u=0$~(m) and $d_u=1.0$~(m).}
    \label{tab:simulation}
\end{table}

For the comparison with PACNav~\cite{ahmad2022pacnav}, we simulate three
scenarios with different number of robots and a varying
number of randomly generated circular obstacles. In the simulation, we do not
inject noise in the measurements, nevertheless tracking error is present. For
each scenario, we perform $10$ runs for statistical evidence. The metrics that
we took into account are the time required for all the robots to travel
$42$~meters in the goal direction and the success rate in terms of mutual
collision avoidance and obstacle avoidance. In Table~\ref{tab:simulation1}, we
report the results obtained from PACNav~\cite{ahmad2022pacnav} and from our
approach. We considered the following scenarios: $N=[4,9,9]$, $\delta_i =
[0.25,0.5,0.5]$~(m), $\delta_o =[0.3, 0.75, 0.75]$~(m), and we select the
parameters as follows for all the three scenarios $r_{s,i} = 4.5$~(m),
$d_1=d_3=0.5$~(m), $d_2=d_4=1$~(m), $\beta_i^D=0.5$, $\Gamma_{ij} = 8$~(m), $d_u
= 0.8$~(m). The adjacency matrices defining the proximity maintenance
constraints were selected as grid adjacency matrices.

For our approach, we considered two different values for the parameters
$\epsilon = \epsilon_p = \epsilon_o$. The parameter selection for
PACNav~\cite{ahmad2022pacnav} has been done by trial and error approach, on the
basis of the environment and the number of robots. Besides that, as seen in
Table~\ref{tab:simulation1}, PACNav does not scale in terms of performance
(convergence, safety, and proximity maintenance) when the number of robots
grows, due to the complex superimposition of forces given by the surroundings.

\begin{table}
    \centering
    \renewcommand{\tabcolsep}{0.1cm} 
    \begin{tabular}{|c|c|c|c|c|c|c|c|}
    \hline
      & &\multicolumn{2}{c|}{PACNav~\cite{ahmad2022pacnav}} &
      \multicolumn{2}{c|}{Our ($\epsilon = 2.00$)} & \multicolumn{2}{c|}{Our
      ($\epsilon = 1.052$)} \\ \hline
      Scenario & N & avg. time (s) & SR & avg. time (s) & SR & avg. time (s) & SR \\ \hline
      1 & 3 & 55.01 & 1.00 & 39.96 & 1.00 & 24.23 & 1.00 \\ \hline
      2 & 9 & - & 0.00 & 45.71 & 1.00 & 29.53 & 1.00 \\ \hline
      3 & 9 & - & 0.00 & 120.33 & 1.00 & 55.87 & 1.00 \\ \hline
    \end{tabular}
    \vspace{10pt} 
    \caption{The simulation results obtained in terms of time to
    reach the goal and success rate for $10$ instances in $3$ different
    scenarios, comparing PACNav~\cite{ahmad2022pacnav} with our method in two
    different parameter configurations.}
    \label{tab:simulation1}
\end{table}

For further simulation results, please refer to the multimedia
material.

\section{Experimental Results}
\label{sec:Experimental results}
To validate our algorithm, we implemented it on real robotic platforms. 
To detect the surrounding obstacles we relied on
the RPLiDAR A3. In order to have a robust and safer perception system, we
introduced a double layer of obstacles. The first layer of obstacles is given by
the occupancy map built online by Hector~\cite{hector_slam}, while the second
layer is given by the raw pointcloud provided by the RPLiDAR A3. Given this
information, we cluster the points on the basis of their spatial distribution.
Then we compute the closest point to the robot position for each cluster, in
order to estimate the position of the obstacles in the scene. By doing that, we
exploit both the reactiveness of the raw data coming from the LiDAR pointcloud,
and at the same time, the scan-matching based obstacle representation, which is
provided by the Hector mapping algorithm. At the same time we do not need to
know the obstacles' size, since we are considering the closest point, we can
account for a small value for $\delta^o$.
The relative positions of the neighboring robots are obtained by relying on the
\ac{UVDAR} technology~\cite{walter2019uvdar}. The sensor system operates by
detection of blinking ultraviolet LEDs. The blinking signal is retrieved by
on-board cameras with tailored filters. Each UAV is individually identifiable
according to its blinking frequency. The pose of the UAVs is inferred through
computer vision techniques.

The experimental deployment is conducted in the forest, where we do not have
access to the GNSS signal. The goal positions are set to be $40$~(m) ahead with
respect to the initial position of each robot, notice that we do not need to
have a common reference frame. We choose to adopt a fixed value of $d_u$
in the experiments. Due to the high measurements' uncertainties the value of
$d_u$ did not satisfy~\eqref{eq:d_inequality} all the time. In this regard, two
practical solutions were adopted in the experiments: (i) We discard measurements
with values of $\lambda^{\max}>15$, and (ii) if the measured distance between
$i$-th and $j$-th robot with $j \in \bar{\mathcal{N}_i}$ is greater than
$\Gamma_{ij}$, we re-project the position of the $j$-th robot to be at a
distance less than $\Gamma_{ij}$ from the $i$-th robot, in order to ensure the
existence of a nonempty set $\mathcal{F}_i$. These two strategies help to
practically manage large estimation errors that exceed the theoretical threshold
e.g., temporary failures in the estimation of distance to other robot's due
to tree occlusions.

We tested the algorithm with up to $N=4$, by selecting $\delta_i = 0.3$ (m),
$\delta^o = 0.05$ (m), $r_{s,i} = 6$ (m), $d_1=d_3=0.5$~(m), $d_2=d_4=1$~(m),
$\beta_i^D=2.5$, $\Gamma_{ij} = 12$~(m), $d_u = 0.8$~(m), $\epsilon_p =
\epsilon_o = 1.052$, and the adjacency matrix that defines the flocking
constraints is $A_d = \left[\begin{smallmatrix}
0&1&0&1\\1&0&1&0\\0&1&0&1\\1&0&1&0 \end{smallmatrix}\right]$. 

In Figure~\ref{fig:map}-(a) we show the trajectory of UAV~$2$, the positions
of the other robots measured by UVDAR system, and the map built by UAV~$2$ by
using Hector SLAM during the mission.
In Figure~\ref{fig:distancesuvdar}, we depict the distances between UAV~$2$ and
the other robots according to the UVDAR measurements. In Table~\ref{tab:vel1}, we
report the average velocities, the maximum velocities, and the time to reach
the goal for each UAV during the mission.
To verify the success of the mission, we report in Figure~\ref{fig:map}-(b) the
actual trajectories for all the 4 robots in the mission, while in
Figure~\ref{fig:distances4} we show the distances between the UAVs during the
mission. The safety constraints and the flocking constraints
(according to the adjancency matrix $A_d$) were satisfied during the entire
mission.

To show reliability and robustness to parameter changes, we report in
Table~\ref{tab:exp} the results obtained from other four experiments in the
forest, with $N=3$ UAVs. Since the system is intrinsically
chaotic it does not make sense to compare the results obtained, even if, as we
can expect, a relation between $\beta^D$ and the time to reach the goal can be
inspected from Table~\ref{tab:exp}. Nevertheless, through the table we want to
show the repeatability of the experiments and the robustness to parameter
changes. The parameters for each experiment were set as follows: 
$\#$experiment $=[1,2,3,4]$, $N=[3,3,3,3]$, $\delta_i = [0.3,0.3,0.3,0.3]$~(m),
$\delta^o = [0.05,0.05,0.05,0.05]$~(m) $r_{s,i} = [6,6,6,6]$~(m),
$d_1=d_3=[0.5,0.5,0.5,0.5]$~(m), $d_2=d_4=[1,1,1,1]$~(m),
$\beta_i^D=[2.0,2.0,1.0,0.6]$, $\Gamma_{ij} = [12,10,10,10]$~(m), $d_u =
[0.75,0.75,0.75,0.75]$~(m), $\epsilon_p = \epsilon_o =
[1.052,1.052,1.010,1.052]$, and the adjacency matrix that defines the flocking
constraints is the same for all the four experiments, $A_d =
\left[\begin{smallmatrix} 0&0&1\\0&0&1\\1&1&0  \end{smallmatrix}\right].$ 

The videos of the experiments discussed in this section can be found in the
multimedia material.

\begin{table}
    \centering
    \renewcommand{\tabcolsep}{0.35cm} 
    \begin{tabular}{|c|c|c|c|c|}
    \hline
      UAV & UAV1 & UAV2 & UAV3 & UAV4 \\ \hline
      avg. speed (m/s) &0.361 & 0.329  & 0.373   & 0.370 \\ \hline
      max. speed (m/s) &1.256 &  1.224 & 1.427   & 1.095\\ \hline
      time (s)         &150.61 & 174.95 & 159.20  & 174.20 \\ \hline
    \end{tabular}
    \vspace{10pt}
    \caption{Average speed, maximum speed, and time to reach the goal of each 
    UAV in the experiment from Figure~\ref{fig:map}.} \label{tab:vel1}
\end{table}

\begin{figure}[t]
 \centering
 \setlength{\tabcolsep}{0.05em}
 \begin{tabular}{cc}
   \includegraphics[width=0.45\linewidth]{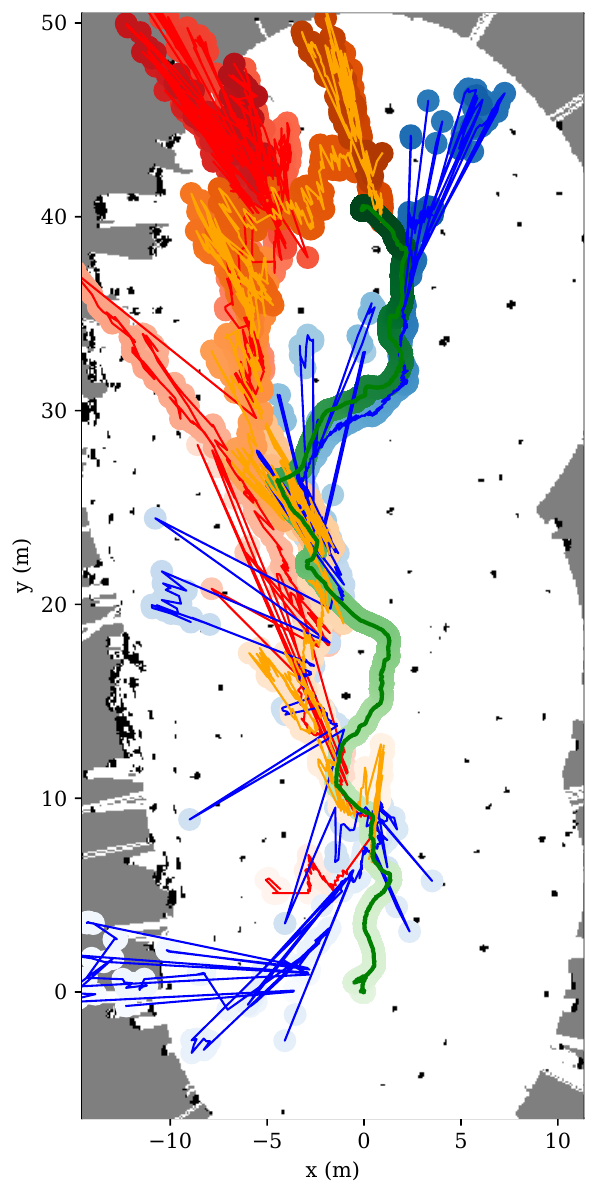} &
  \includegraphics[width=0.4415\linewidth]{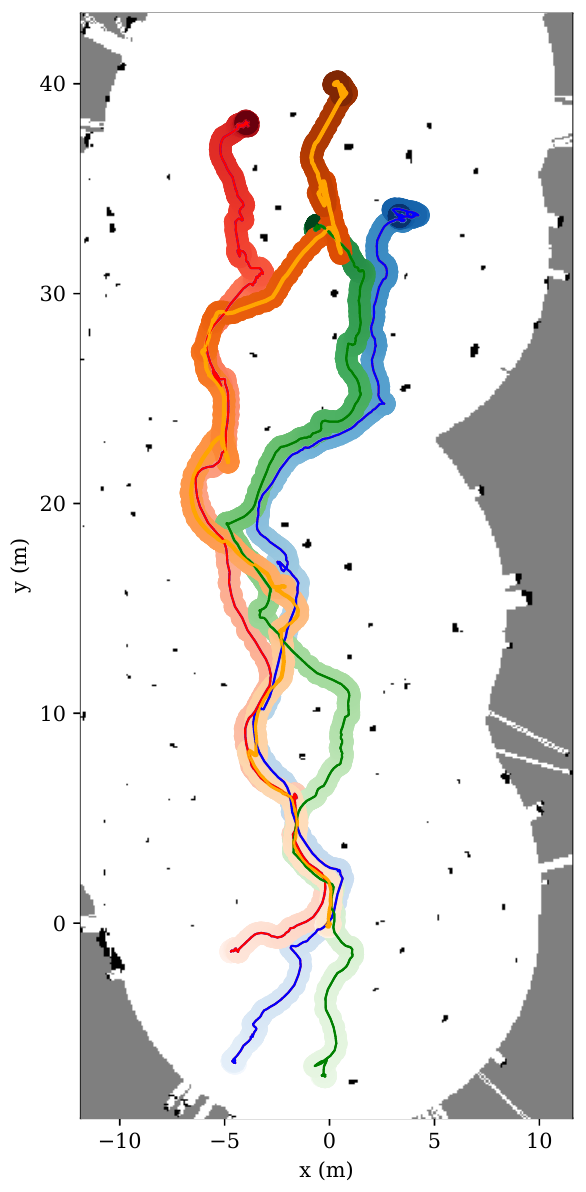}\\
   (a) & (b)
   \end{tabular}
  \caption{Experiment in the forest. In (a) the trajectory of UAV~$2$ is
  depicted in green. The positions measured through UVDAR by UAV~$2$ are
  depicted in orange (UAV~$1$), blue (UAV~$3$), and red (UAV~$4$). The map built
  by UAV~$2$ by means of Hector mapping is also represented, in particular the
  free space is white, obstacles are black and the unexplored area is grey. In
  (b) we depicted the actual trajectories obtained in post-process phase through
  transformation of odometries to a common reference frame. The map was built by
  UAV~$1$. }
  \label{fig:map}
\end{figure}

\begin{figure}
    \centering
   \includegraphics[width=0.9\linewidth]{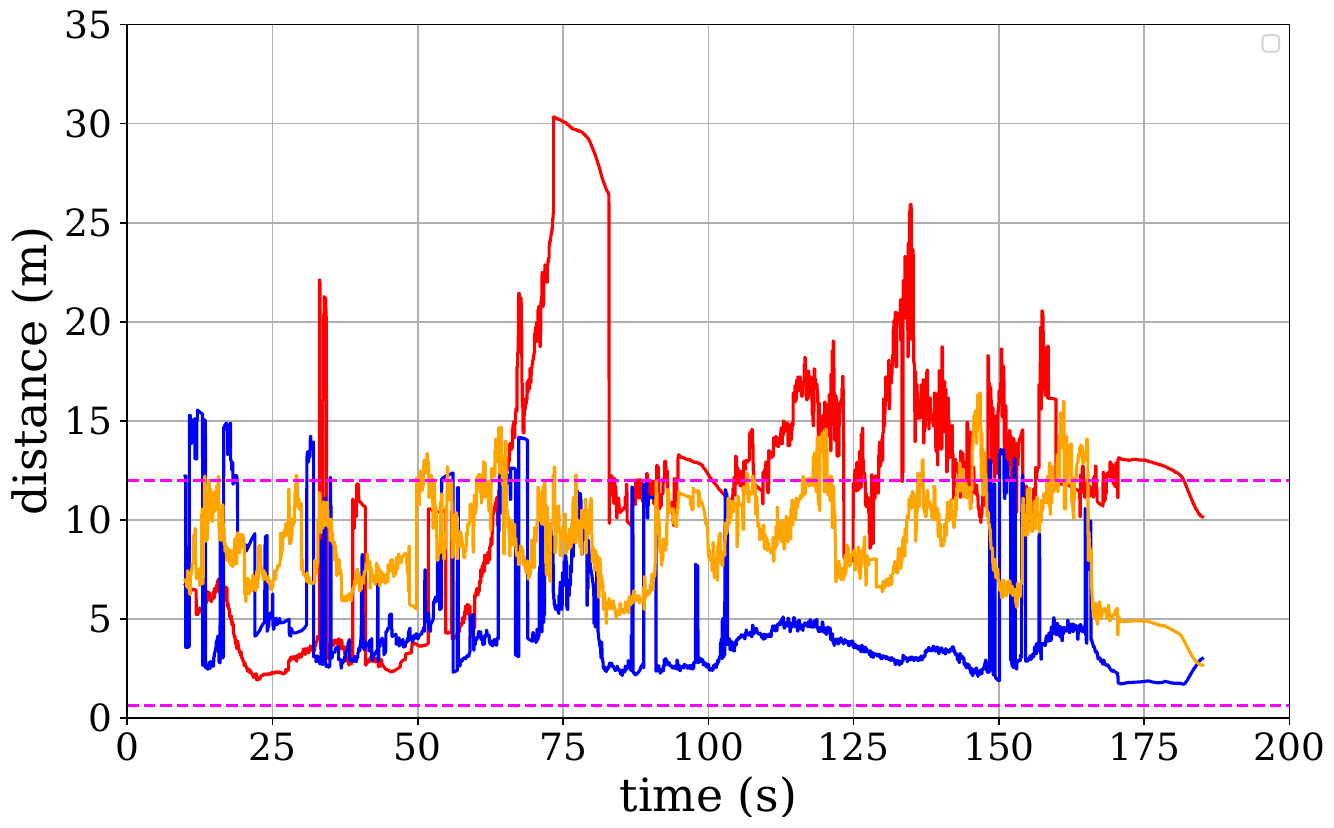}
    \caption{Distances measured by UVDAR system on UAV 2 from UAV 1 (blue), UAV
    3 (orange), and UAV 4 (red). The flocking and safety constraints are in
    magenta.}
    \label{fig:distancesuvdar}
\end{figure}

\begin{figure}
    \centering
   \includegraphics[width=0.9\linewidth]{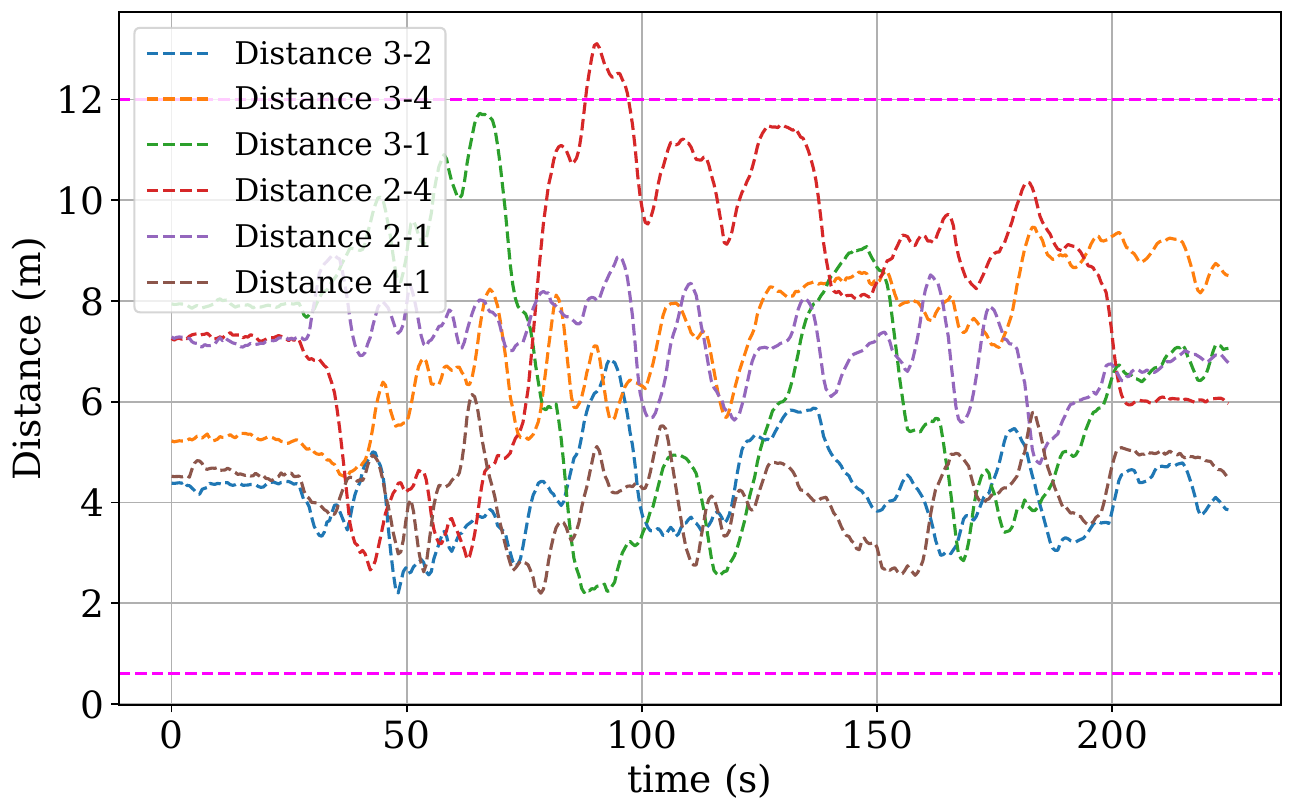}
    \caption{Distances between UAVs computed based on odometries transformed
    to a common reference frame obtained through alignment of maps of
    individual UAVs in a post-prossessing phase. The flocking and safety
    constraints are in magenta. According to $A_d$ the flocking
    constraints are not active for the pair of robots $2$-$4$ and $1$-$3$.}
    \label{fig:distances4}
\end{figure}

\begin{table}
    \centering
    \renewcommand{\tabcolsep}{0.55cm} 
    \begin{tabular}{|c|c|c|c|}
    \hline
      &UAV1 & UAV2 & UAV3 \\ \hline
      \multicolumn{4}{|c|}{Exp. 1 ($\beta^D=2.0$, $\epsilon_{o(p)}=1.052$, $\Gamma_{ij}=12$)}\\ \hline
      avg. speed (m/s)& 0.309 & 0.249 & 0.263\\ \hline
      max. speed (m/s)& 1.210 & 1.255 & 1.226\\ \hline
      time (s) & 278.43&280.54&245.12 \\ \hline
      \multicolumn{4}{|c|}{Exp. 2 ($\beta^D=2.0$, $\epsilon_{o(p)}=1.052$, $\Gamma_{ij}=10$)}\\ \hline
      avg. speed (m/s)& 0.447& 0.415& 0.382\\ \hline
      max. speed (m/s)& 1.150& 1.170& 1.230\\ \hline
      time (s) & 148.00 & 163.81 & 159.91 \\ \hline
      \multicolumn{4}{|c|}{Exp. 3 ($\beta^D=1.0$, $\epsilon_{o(p)}=1.010$, $\Gamma_{ij}=10$)}\\ \hline
      avg. speed (m/s)& 0.661& 0.497& 0.740\\ \hline
      max. speed (m/s)& 1.229& 1.245& 1.374\\ \hline
      time (s) & 85.16 & 117.16 & 79.20 \\ \hline
      \multicolumn{4}{|c|}{Exp. 4 ($\beta^D=0.6$, $\epsilon_{o(p)}=1.052$, $\Gamma_{ij}=10$)}\\ \hline
      avg. speed (m/s)& 0.605& 0.651& 0.794\\ \hline
      max. speed (m/s)& 1.175& 1.281& 1.310\\ \hline
      time (s) & 82.49 & 80.33 & 62.30\\ \hline
    \end{tabular}
    \vspace{10pt} \caption{Comparison of average speed, maximum speed, and time
    to reach the goal in four different experiments in the forest, showing
    repeatability and robustness to parameters changes.} \label{tab:exp}
\end{table}

\section{Conclusions} \label{sec:Conclusions} 

In this paper, we presented a novel distributed algorithm to swarm reliably and
effectively in dense cluttered environments represented by a forest. We provide
sufficient conditions for safety and proximity maintenance. We propose an
adaptive strategy to compensate for tracking errors and measurement
uncertainties. We compare our approach with state-of-the-art algorithms, and
finally we validate our solution through simulations and experiments in the
field and in dense environment such as a forest. Our approach presents some
limitations; the algorithm is limited to work on a plane, i.e., the reference in
the altitude for each robot is constant and equal. The convergence towards
the goal depends on both the environment of operation and the uncertainties
taken into account. Providing to the algorithm a higher level re-planner can be
beneficial to enhance performance. Future research directions will be focused on
an effective 3D extension of the Lloyd algorithm and to further increase the
agility of the swarm.

\begin{acronym}
  \acro{CNN}[CNN]{Convolutional Neural Network}
  \acro{IR}[IR]{infrared}
  \acro{GNSS}[GNSS]{Global Navigation Satellite System}
  \acro{MOCAP}[mo-cap]{Motion capture}
  \acro{MPC}[MPC]{Model Predictive Control}
  \acro{MRS}[MRS]{Multi-robot Systems Group}
  \acro{ML}[ML]{Machine Learning}
  \acro{MAV}[MAV]{Micro-scale Unmanned Aerial Vehicle}
  \acro{UAV}[UAV]{Unmanned Aerial Vehicle}
  \acro{UV}[UV]{ultraviolet}
  \acro{UVDAR}[\emph{UVDAR}]{UltraViolet Direction And Ranging}
  \acro{UT}[UT]{Unscented Transform}
  \acro{RTK}[RTK]{Real-Time Kinematic}
  \acro{ROS}[ROS]{Robot Operating System}
  \acro{wrt}[w.r.t.]{with respect to}
  \acro{FEC}[FEC]{formation-enforcing control}
  \acro{DIFEC}[DIFEC]{distributed formation-enforcing control}
  \acro{LIDAR}[LiDAR]{Light Detection And Ranging}
  \acro{UWB}[UWB]{Ultra-wideband}
\end{acronym}

\bibliographystyle{IEEEtran}
\bibliography{reference}

@article{boldrer2022unified, title={A unified Lloyd-based framework for
  multi-agent collective behaviours}, author={Boldrer, Manuel and Palopoli,
  Luigi and Fontanelli, Daniele}, journal={Robotics and Autonomous Systems},
  volume={156}, pages={104207}, year={2022}, publisher={Elsevier} }

@article{zhou2022swarm, title={Swarm of micro flying robots in the wild},
      author={Zhou, Xin and Wen, Xiangyong and Wang, Zhepei and Gao, Yuman and
        Li, Haojia and Wang, Qianhao and Yang, Tiankai and Lu, Haojian and Cao,
        Yanjun and Xu, Chao and others}, journal={Science Robotics},
      volume={7}, number={66}, pages={eabm5954}, year={2022},
      publisher={American Association for the Advancement of Science} }

@article{boldrer2023rule, title={Rule-Based Lloyd Algorithm for
        Multi-Robot Motion Planning and Control with Safety and Convergence
          Guarantees}, author={Boldrer, Manuel and Serra-Gomez, Alvaro and
            Lyons, Lorenzo and Kr{\'a}tk{\'y}, V{\'i}t and Alonso-Mora, Javier and Ferranti, Laura},
        journal={arXiv preprint arXiv:2310.19511}, year={2024} }

@article{petravcek2020bio, title={Bio-inspired compact swarms of unmanned
      aerial vehicles without communication and external localization},
      author={Petr{\'a}{\v{c}}ek, Pavel and Walter, Viktor and B{\'a}{\v{c}}a,
        Tom{\'a}{\v{s}} and Saska, Martin}, journal={Bioin. \&
          Biomim.}, volume={16}, number={2}, pages={026009}, year={2020},
        publisher={IOP Publishing} }

@article{lloyd1982least, title={Least
          squares quantization in PCM}, author={Lloyd, Stuart}, journal={IEEE
            transactions on information theory}, volume={28}, number={2},
          pages={129--137}, year={1982}, publisher={IEEE} }

@article{cortes2004coverage, title={Coverage control for mobile
            sensing networks}, author={Cortes, Jorge and Martinez, Sonia and
              Karatas, Timur and Bullo, Francesco}, journal={IEEE Transactions
                on robotics and Automation}, volume={20}, number={2},
            pages={243--255}, year={2004}, publisher={IEEE} }

@inproceedings{boldrer2020lloyd, title={Lloyd-based approach for
              robots navigation in human-shared environments}, author={Boldrer,
                Manuel and Palopoli, Luigi and Fontanelli, Daniele},
              booktitle={2020 IEEE/RSJ Intern. Conf. on Intell.
                Robots and Systems (IROS)}, pages={6982--6989}, year={2020},
               }

@article{boldrer2021graph, title={Graph
                connectivity control of a mobile robot network with mixed
                  dynamic multi-tasks}, author={Boldrer, Manuel and Bevilacqua,
                    Paolo and Palopoli, Luigi and Fontanelli, Daniele},
                journal={IEEE Robotics and Automation Letters}, volume={6},
                number={2}, pages={1934--1941}, year={2021}, publisher={IEEE} }

@article{ahmad2022pacnav,
                    title={PACNav: a collective navigation approach for UAV
                      swarms deprived of communication and external
                        localization}, author={Ahmad, Afzal and others},
                    journal={Bioin. \& Biomim.}, volume={17},
                    number={6}, pages={066019}, year={2022}, publisher={IOP
                      Publishing} }

@ARTICLE{toumieh2024high,
    author={Toumieh, Charbel and Floreano, Dario},
    journal={IEEE Transactions on Robotics}, 
    title={High-Speed Motion Planning for Aerial Swarms in Unknown and Cluttered Environments}, 
    year={2024},
    volume={40},
    number={},
    pages={3642-3656},
    doi={10.1109/TRO.2024.3429193}}

@article{vasarhelyi2018optimized, title={Optimized flocking of autonomous
  drones in confined environments}, author={V{\'a}s{\'a}rhelyi, G{\'a}bor and
    Vir{\'a}gh, Csaba and Somorjai, Gerg{\H{o}} and Nepusz, Tam{\'a}s and
      Eiben, Agoston E and Vicsek, Tam{\'a}s}, journal={Science Robotics},
  volume={3}, number={20}, pages={eaat3536}, year={2018}, publisher={American
    Association for the Advancement of Science} }

@inproceedings{adajania2023amswarm, title={Amswarm: An alternating minimization
  approach for safe motion planning of quadrotor swarms in cluttered
    environments}, author={Adajania, Vivek K and Zhou, Siqi and Singh, Arun
      Kumar and Schoellig, Angela P}, booktitle={2023 IEEE International
        Conference on Robotics and Automation}, pages={1421--1427},
  year={2023} }

@inproceedings{zhou2021ego, title={Ego-swarm: A fully autonomous and
  decentralized quadrotor swarm system in cluttered environments},
  author={Zhou, Xin and Zhu, Jiangchao and Zhou, Hongyu and Xu, Chao and Gao,
    Fei}, booktitle={2021 IEEE international conference on robotics and
      automation (ICRA)}, pages={4101--4107}, year={2021}
}

@inproceedings{boldrer2020socially, title={Socially-aware
               multi-agent velocity obstacle based navigation for nonholonomic
                 vehicles}, author={Boldrer, Manuel and Palopoli, Luigi and
                   Fontanelli, Daniele}, booktitle={2020 IEEE 44th Annual
                     Computers, Software, and Applications Conference
                       (COMPSAC)}, pages={18--25}, year={2020},
                 organization={IEEE} }

@article{zhu2022decentralized,
                   title={Decentralized probabilistic multi-robot collision
                     avoidance using buffered uncertainty-aware Voronoi cells},
                   author={Zhu, Hai and Brito, Bruno and Alonso-Mora, Javier},
                   journal={Autonomous Robots}, volume={46}, number={2},
                   pages={401--420}, year={2022}, publisher={Springer} }

@article{ferranti2022distributed, title={Distributed nonlinear trajectory
  optimization for multi-robot motion planning}, author={Ferranti, Laura and
    Lyons, Lorenzo and Negenborn, Rudy R and Keviczky, Tam{\'a}s and
      Alonso-Mora, Javier}, journal={IEEE Transactions on Control Systems
        Technology}, volume={31}, number={2}, pages={809--824}, year={2022},
  publisher={IEEE} }

@article{walter2019uvdar, title={Uvdar system for visual relative localization
  with application to leader--follower formations of multirotor uavs},
  author={Walter, Viktor and Staub, Nicolas and Franchi, Antonio and Saska,
    Martin}, journal={IEEE Robotics and Automation Letters}, volume={4},
  number={3}, pages={2637--2644}, year={2019}, publisher={IEEE} }

@article{baca2021mrs, title={The MRS UAV system: Pushing the frontiers of
  reproducible research, real-world deployment, and education with autonomous
    unmanned aerial vehicles}, author={Baca, Tomas and others}, journal={Journal of Intelligent \& Robotic Systems},
             volume={102}, number={1}, pages={26}, year={2021},
             publisher={Springer} }

@article{walter2023distributed, title={Distributed formation-enforcing control
  for UAVs robust to observation noise in relative pose measurements},
  author={Walter, Viktor and Vrba, Matou{\v{s}} and Licea, Daniel Bonilla and
    Hilmer, Matej and Saska, Martin}, journal={arXiv preprint
      arXiv:2304.03057}, year={2023} }

@inproceedings{vasarhelyi2014outdoor, title={Outdoor flocking and formation
  flight with autonomous aerial robots}, author={V{\'a}s{\'a}rhelyi, G{\'a}bor
    and Vir{\'a}gh, Cs and Somorjai, Gergo and Tarcai, Norbert and
      Sz{\"o}r{\'e}nyi, Tam{\'a}s and Nepusz, Tam{\'a}s and Vicsek, Tam{\'a}s},
    booktitle={2014 IEEE/RSJ International Conference on Intelligent Robots and
      Systems}, pages={3866--3873}, year={2014}}

@inproceedings{wasik2016graph, title={Graph-based distributed control for
  adaptive multi-robot patrolling through local formation transformation},
  author={Wasik, Alicja and Pereira, Jos{\'e} N and Ventura, Rodrigo and Lima,
    Pedro U and Martinoli, Alcherio}, booktitle={IEEE/RSJ International
      Conference on Intelligent Robots and Systems}, pages={1721--1728},
    year={2016}}

@inproceedings{soares2015distributed, title={A distributed formation-based odor
  source localization algorithm-design, implementation, and wind tunnel
    evaluation}, author={Soares, Jorge M and Aguiar, A Pedro and Pascoal,
      Ant{\'o}nio M and Martinoli, Alcherio}, booktitle={IEEE
        International Conference on Robotics and Automation},
      pages={1830--1836}, year={2015} }

@article{mezey2024purely,
      title={Purely vision-based collective movement of robots},
      author={Mezey, David and Bastien, Renaud and Zheng, Yating and McKee, Neal and Stoll, David and Hamann, Heiko and Romanczuk, Pawel},
      journal={npj Robotics},
      volume={3},
      number={1},
      pages={11},
      year={2025},
      publisher={Nature Publishing Group UK London}
    }

@inproceedings{hauert2011reynolds, title={Reynolds flocking in reality with
  fixed-wing robots: communication range vs. maximum turning rate},
  author={Hauert, Sabine and Leven, Severin and Varga, Maja and Ruini, Fabio
    and Cangelosi, Angelo and Zufferey, Jean-Christophe and Floreano, Dario},
  booktitle={IEEE/RSJ international conference on intelligent robots and
    systems}, pages={5015--5020}, year={2011}}

@inproceedings{quintero2013flocking, title={Flocking with fixed-wing UAVs for
  distributed sensing: A stochastic optimal control approach},
  author={Quintero, Steven AP and Collins, Gaemus E and Hespanha, Joao P},
  booktitle={IEEE American Control Conference}, pages={2025--2031},
  year={2013} }

@inproceedings{turpin2012decentralized, title={Decentralized formation control
  with variable shapes for aerial robots}, author={Turpin, Matthew and Michael,
    Nathan and Kumar, Vijay}, booktitle={2012 IEEE international conference on
      robotics and automation}, pages={23--30}, year={2012} }

@misc{hector_slam,
  author       = {Stefan Kohlbrecher and Johannes Meyer},
  title        = {Hector SLAM: A Flexible and High-Performance SLAM System for Mobile Robots},
  howpublished = {\url{https://wiki.ros.org/hector_mapping}},
  year         = {2013},
}

@ARTICLE{zhu2024swarm,
    author={Zhu, Fangcheng and Ren, Yunfan and Yin, Longji and Kong, Fanze and Liu, Qingbo and Xue, Ruize and Liu, Wenyi and Cai, Yixi and Lu, Guozheng and Li, Haotian and Zhang, Fu},
    journal={IEEE Transactions on Robotics}, 
    title={Swarm-LIO2: Decentralized Efficient LiDAR-Inertial Odometry for Aerial Swarm Systems}, 
    year={2025},
    volume={41},
    number={},
    pages={960-981},
    doi={10.1109/TRO.2024.3522155}}

@article{bartolomei2023fast,
  title={Fast multi-UAV decentralized exploration of forests},
  author={Bartolomei, Luca and Teixeira, Lucas and Chli, Margarita},
  journal={IEEE Robotics and Automation Letters},
  year={2023},
  publisher={IEEE}
}

@article{soria2021predictive,
  title={Predictive control of aerial swarms in cluttered environments},
  author={Soria, Enrica and Schiano, Fabrizio and Floreano, Dario},
  journal={Nature Machine Intelligence},
  volume={3},
  number={6},
  pages={545--554},
  year={2021},
  publisher={Nature Publishing Group UK London}
}

@article{tordesillas2021mader,
  title={MADER: Trajectory planner in multiagent and dynamic environments},
  author={Tordesillas, Jesus and How, Jonathan P},
  journal={IEEE Transactions on Robotics},
  volume={38},
  number={1},
  pages={463--476},
  year={2021},
  publisher={IEEE}
}

@article{park2025decentralized,
  title={Decentralized trajectory planning for quadrotor swarm in cluttered environments with goal convergence guarantee},
  author={Park, Jungwon and Lee, Yunwoo and Jang, Inkyu and Kim, H Jin},
  journal={The International Journal of Robotics Research},
  pages={1336--1359},
  year={2025},
  publisher={SAGE Publications Sage UK: London, England}
}

\end{document}